\documentclass[default]{sn-jnl}

\usepackage{graphicx}%
\usepackage{multirow}%
\usepackage{amsmath,amssymb,amsfonts}%

\usepackage[title]{appendix}%
\usepackage{xcolor}%
\usepackage{manyfoot}%
\usepackage{booktabs}%
\usepackage{listings}%

\usepackage{url}
\usepackage{bm}
\usepackage{subcaption}
\usepackage{orcidlink}
\usepackage{enumitem}
\usepackage{rotating, tabularx}
\usepackage[titlenumbered, ruled, linesnumbered, noend]{algorithm2e}
\SetArgSty{textup}

\SetCommentSty{mycommfont}
\SetKw{And}{and}
\SetKw{Break}{break}
\SetKw{Or}{or}
\SetKw{In}{in}
\SetKw{Continue}{continue}
\let\oldnl\nl
\newcommand{\nonl}{\renewcommand{\nl}{\let\nl\oldnl}}
\usepackage{array}
\usepackage{pifont}
\usepackage{makecell,}
\usepackage{adjustbox}
\usepackage{tabularx}
\newcommand{\cmark}{\ding{51}}
\newcommand{\xmark}{-}%
\definecolor{light-gray}{gray}{0.95}
\usepackage{fnpct}
\usepackage{tablefootnote}
\usepackage{natbib}
\setcitestyle{authoryear}
\setcitestyle{round}
\setcitestyle{semicolon}
\setcitestyle{aysep={}}
\usepackage{amsthm}
\newtheorem{theorem}{Theorem}
\newtheorem{lemma}[theorem]{Lemma}

\theoremstyle{definition}

\newcommand*{\ngt}{g}
\newcommand*{\nd}{d}
\newcommand*{\nmin}{m}
\usepackage[textsize=scriptsize]{todonotes}
\usepackage{marginnote}

\usepackage{cancel}
\usepackage[normalem]{ulem}
\usepackage{soul}
\usepackage{xcolor}
\definecolor{tabblue}{RGB}{31, 93, 180}
\definecolor{taborange}{RGB}{255, 127, 14}

\begin{document}

\title[Article Title]{Quantitative Evaluation of Motif Sets in Time Series}

\author*[1,2]{\fnm{Daan} \sur{Van Wesenbeeck} \orcidlink{0000-0002-4941-5480}}\email{daan.vanwesenbeeck@kuleuven.be}
\author[1,2]{\fnm{Aras} \sur{Yurtman} \orcidlink{0000-0001-6213-5427}}
\author[1,2]{\fnm{Wannes} \sur{Meert} \orcidlink{0000-0001-9560-3872}}
\author[1,2]{\fnm{Hendrik} \sur{Blockeel} \orcidlink{0000-0003-0378-3699}}

\affil[1]{\orgdiv{Dept. of Computer Science}, \orgname{KU Leuven}, \orgaddress{\city{Leuven}, \postcode{B-3000}, \country{Belgium}}}
\affil[2]{\orgdiv{Leuven.AI - KU Leuven Institute for AI}, \city{Leuven}, \postcode{B-3000}, \country{Belgium}}

\abstract{

    Time Series Motif Discovery (TSMD), which aims at finding recurring patterns in time series, is an important task in numerous application domains, and many methods for this task exist. These methods are usually evaluated qualitatively. A few metrics for quantitative evaluation, where discovered motifs are compared to some ground truth, have been proposed, but they typically make implicit assumptions that limit their applicability. This paper introduces PROM, a broadly applicable metric that overcomes those limitations, and TSMD-Bench, a benchmark for quantitative evaluation of time series motif discovery. Experiments with PROM and TSMD-Bench show that PROM provides a more comprehensive evaluation than existing metrics, that TSMD-Bench is a more challenging benchmark than earlier ones, and that the combination can help understand the relative performance of TSMD methods. More generally, the proposed approach enables large-scale, systematic performance comparisons in this field.
    
}

\keywords{Evaluation metrics, benchmark datasets, time series, motif discovery}

\maketitle

\section{Introduction}\label{sec:intro}

Time series motif discovery (TSMD) has emerged as a fundamental time series task. It is useful in a wide variety of domains, such as medicine \citep{tsmd_medicine}, robotics \citep{tsmd_robotics}, seismology \citep{tsmd_seismo}, audio processing \citep{tsmd_audio}, and so on.
It consists of identifying recurring patterns, in the form of segments that are similar to each other. Each such segment represents an occurrence of the pattern. In this paper, we call one such segment a \emph{motif}, and the set of all occurrences of a pattern a \emph{motif set}. 

Many different methods for TSMD have been proposed.  Their performance is typically evaluated in a qualitative manner, by having a domain expert interpret the discovered motifs. While this is useful, it does not allow for systematic large-scale performance comparisons between methods. A variety of quantitative methods have occasionally been used as well, but they typically assume a specific setting of TSMD and have limitations that preclude their use in other settings.  

In this paper, we propose PROM, a broadly applicable method for quantitative evaluation of TSMD. It works in the following setting: given a ``ground truth'' consisting of one or more motif sets, and a set of motif sets discovered by a TSMD method, evaluate how well the discovered motif sets approximate the ground truth.  Figure~\ref{fig:example} illustrates this setting. 
It shows a time series in which two ground-truth motif sets can be identified.  A TSMD method has returned three discovered motif sets, which should be evaluated by comparing them to the ground-truth motif sets. The evaluation should take into account the accuracy with which individual motifs are found, as well as the extent to which the discovered motif sets match with the ground-truth motif sets. In the example, one discovered motif set is entirely extraneous, another contains false-positive motifs: both are undesired and should lead to a lower score. 

\begin{figure}
    \centering
    \includegraphics[width=\textwidth]{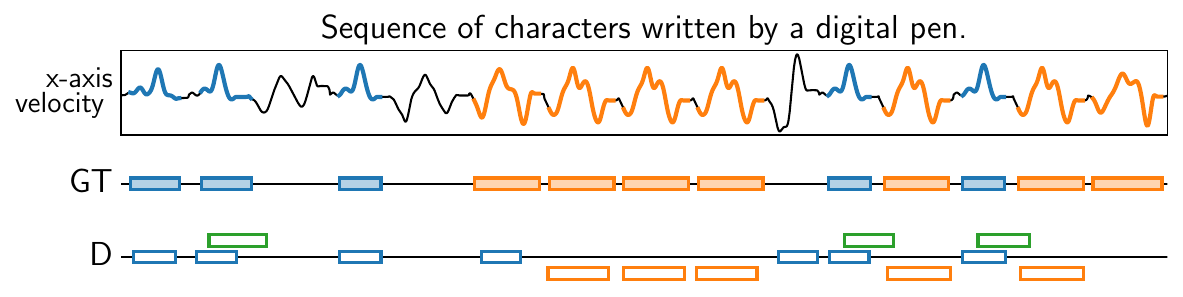}
    \caption{The ground-truth (GT) and discovered (D) motif sets in an example time series that represents a sequence of characters written by a digital pen \citep{misc_character_trajectories_175}, where the goal is to retrieve the occurrences of each character that is repeated. A \emph{quantitative evaluation} scores how well the discovered motif sets correspond to the GT motif sets. In this example, the orange discovered motif set is incomplete with respect to the orange GT motif set, the blue discovered motif set has excess motifs, and the green one is redundant given the blue one.}
    \label{fig:example}
\end{figure}

Existing methods for quantitative evaluation often (implicitly) make assumptions such as: all motifs have the same length; the ground truth contains exactly one motif set; motif sets contain exactly two motifs; false-positives motifs need not be penalized. All these assumptions are violated in this example case. The evaluation criterion proposed here is quite simple and in fact has already been used at least once \citep{locomotif}, but without extensive motivation or explanation of its properties.  We here study this criterion in detail and show that it enables systematic large-scale quantitative evaluation of TSMD methods.

Of course, the proposed type of evaluation requires benchmark time series in which the ground-truth motif sets are known. Since real, expert-annotated datasets are scarcely available, existing studies often use synthetic time series datasets, typically created by inserting short time series instances with similar shape into random-walk data. But retrieving the motifs in such time series is often trivial, because it boils down to distinguishing random from non-random data. This arguably makes these synthetic TSMD tasks less representative of realistic TSMD tasks.
Therefore, as a second contribution, we propose TSMD-Bench: a set of synthetic datasets that are entirely constructed from real time series datasets and should therefore be more representative for real TSMD tasks.

The proposed evaluation metric and the benchmark are separate and independent contributions, which may be useful on their own.\footnote{Source code for PROM and TSMD-Bench are available at \texttt{\url{https://github.com/ML-KULeuven/tsdm-evaluation}}.} Combined, they allow for easy and effective evaluation of new TSMD methods.

\section{Problem Statement}\label{sec:problem}

Before stating our goal more precisely, we introduce some notation that will be used throughout the paper.

A \emph{time series} $\mathbf{x}$ is defined as a sequence of $n$ samples from a feature space $\mathcal{F}$: $\mathbf{x} = [x_1, x_2, \dots, x_n]  \quad \text{where} \quad x_i \in \mathcal{F} \ \text{for} \ i = 1, \dots, n$. $\mathcal{F} = \mathbb{R}$ for univariate time series and $\mathcal{F} = \mathbb{R}^D$ for multivariate time series with $D$ dimensions. A \emph{segment} $\alpha$ of a time series $\mathbf{x}$ that starts at index $b$ and ends at index $e$ is denoted as $\alpha = [b:e] = [b, b+1, \dots, e] \quad  \text{where} \quad  1 \leq b \leq e \leq n.$ The \emph{length} of $\alpha$ is denoted as $|\alpha|$ and is equal to $e-b+1$. 

The \emph{ground truth (GT)} motif sets represent the patterns of interest. They are denoted as $\mathcal{G} = \{ \mathcal{G}_{1}, \dots, \mathcal{G}_{\ngt} \}$ where each $\mathcal{G}_{i}$ is a motif set of $k_i$ non-overlapping segments: $\mathcal{G}_{i} = \{\beta_{i, 1}, \dots, \beta_{i, k_i} \}$ for $i=1,\dots,\ngt$. Also segments from different motif sets cannot overlap, i.e., $\beta_{i,j} \cap \beta_{i', j'} = \emptyset$ except when $(i,j)=(i',j')$.  

We assume that the TSMD method that we want to evaluate has returned a set of motif sets $\mathcal{M} = \{ \mathcal{M}_{1}, \dots, \mathcal{M}_{\nd} \}$ where $\mathcal{M}_{j} = \{\alpha_{j, 1}, \dots, \alpha_{j, l_j} \}$ for $j=1, \dots, \nd$. In contrast to the ground-truth motifs, $\mathcal{M}$ may contain overlapping segments.  $\mathcal{M}$ may contain more or fewer motif sets than $\mathcal{G}$, i.e., $\nd$ can be different from $\ngt$. Individual motif sets may contain any number of motifs, which may differ in length.

The task we address is now: given $\mathcal{G}$ and $\mathcal{M}$, describe how well $\mathcal{M}$ approximates $\mathcal{G}$.  Ideally, the evaluation criterion $c(\mathcal{G},\mathcal{M})$ should have the following properties: it is maximal when $\mathcal{G} =\mathcal{M}$, and it is able to penalize: (1) the discovery of motif sets that cannot be matched with ground-truth motif sets (false positives) and vice versa, the existence of ground-truth motif sets that cannot be matched with discovered motif sets (false negatives); (2) discovered motifs in a motif set that cannot be matched with motifs in the matching ground-truth motif set (false positive motifs) and vice versa (false negative motifs). 

\section{Related Work}\label{sec:rw}
\noindent The evaluation methodologies used for TSMD vary considerably across the literature. While it is not the main focus of this paper, we first touch upon qualitative evaluation, as it is the most commonly used.

\subsection{Qualitative Evaluation}
In most qualitative evaluations, one or more TSMD methods are applied to a real time series, and the authors (or ideally domain experts) assign a semantic meaning to the obtained motifs \citep{mp2, mrmotif, mueen_enumeration_nodate, furtado_silva_elastic_2018, alaee_matrix_2020}. Such evaluations are useful to gain new insights about the application domain. In other qualitative evaluations, it is roughly known what the motifs of interest look like (or where they occur in the time series), and the authors simply verify whether they are among the motifs found by the TSMD method \citep{mstamp, motiflets}. While both types of qualitative evaluation demonstrate the practical value of the applied methods, their reliance on human expertise means that they are somewhat subjective, and that they do not scale well to extensive experiments (in which multiple methods are applied to multiple datasets). Quantitative evaluation solves both these issues.


\noindent \subsection{Quantitative Evaluation} 
\begin{table}
\renewcommand{\arraystretch}{1.4}
\newcommand{\rot}[1]{\rotatebox[origin=c]{90}{#1}}
\begin{adjustbox}{width=\textwidth}
\begin{tabular}{p{1.25cm}p{3cm}>{\centering}p{0.5cm}>{\centering}p{0.5cm}>{\centering}p{0.5cm}>{\centering}p{0.5cm}>{\centering}p{0.5cm}>{\centering}p{0.5cm}>{\centering}p{0.5cm}>{\centering}p{0.5cm}|}
&                                                                            & \multicolumn{3}{c}{Setting}       & \multicolumn{4}{c}{Penalizes} \\ \cmidrule(l){3-5} \cmidrule(l){6-9}
Evaluation Metric                    & Reference                             & \rot{\makecell{Variable \\ length}} & \rot{\makecell{Multiple \\motif sets}} & \rot{\makecell{Motif sets of \\ any cardinality}} & \rot{\makecell{False-positive \\ sets}} &  \rot{\makecell{False-negative \\ sets}} & \rot{\makecell{False-positive \\ motifs}} &  \rot{\makecell{False-negative \\ motifs}} \tabularnewline \midrule \\
Overlap-based metrics                & \citet{gao_exploring_2018}            & \cmark          & \cmark              & \xmark                        & \xmark          &  \cmark          & \xmark         &  \cmark  \tabularnewline      
                                     & \citet{nunthanid_parameter-free_2012} & \cmark          & \cmark              & \xmark                        & \xmark          &  \cmark          & \xmark         &  \cmark  \tabularnewline  
                                     & \citet{correctness}                   & \cmark          & \cmark              & \cmark                        & \xmark          &  \xmark          & \xmark         &  \cmark  \tabularnewline                
Score                                & \citet{setfinder}                     & \xmark          & \cmark              & \cmark                        & N/S             & N/S              & \cmark         &  \cmark  \tabularnewline          
Precision and Recall                 & \citet{setfinder, guiding, motiflets} & \xmark          & \xmark              & \cmark                        & N/A             & N/A              & \cmark         &  \cmark  \tabularnewline               
PROM                                 &                                       & \cmark          & \cmark              & \cmark                        & \cmark          & \cmark           & \cmark         &  \cmark  \tabularnewline  \bottomrule    
\end{tabular}
\end{adjustbox}
\caption{Overview of existing TSMD evaluation metrics, which TSMD setting they are defined for, and whether they have the properties listed in the Problem Statement (Section~\ref{sec:problem}). N/A: Not Applicable, N/S: Not Specified.}\label{tab:overview_metrics}
\end{table}
\subsubsection{Existing Evaluation Metrics} 
A small fraction of existing studies use quantitative evaluation metrics: Table \ref{tab:overview_metrics} provides an overview. Since these metrics are introduced together with a specific TSMD method, they are often defined for specific TSMD settings: Some assume fixed-length motifs, a single discovered motif set, or motif pairs (motif sets with exactly two motifs). Others do not penalize the undesired aspects mentioned in Section~\ref{sec:problem}. \\

Several metrics measure the overlap between the discovered motifs and the GT motifs. \citet{gao_exploring_2018} and \citet{nunthanid_parameter-free_2012} define such metrics for motif pairs; \citet{correctness} do so for motif sets of any cardinality. Since they focus solely on the overlap with the GT motifs, they do not penalize false-positive motifs, nor false-positive sets.

The ``score" metric \citep{setfinder} measures differences in time rather than overlap. It measures the differences between the start points of the GT and discovered motifs, and penalizes false-negative motifs and false-positive motifs by adding their length to the score value. The metric is not suitable for evaluating motifs that can considerably differ in length: for a given GT motif, a discovered motif that starts at the same index is considered a perfect match, even if it ends at a completely different index.

The above metrics are applied in the context of multiple GT and discovered motif sets in their respective paper. Since TSMD is an unsupervised task, a method does not explicitly associate a specific discovered motif set with a specific GT motif set. Therefore, these metrics first match every GT motif set with a different discovered motif set, and then aggregate the individual values to obtain a final performance value.
This matching is obtained either in a greedy manner \citep{nunthanid_parameter-free_2012, gao_exploring_2018}, or optimally (such that the final performance value is optimized) \citep{setfinder}. 

Other studies define precision and recall for TSMD \citep{guiding, motiflets, setfinder}, but only for a single GT and discovered motif set that contain fixed-length motifs. First, every discovered motif is \emph{matched} to a GT motif if it is
sufficiently similar in time.\footnote{The definition of ``sufficiently'' differs among studies: sometimes the motifs have to overlap more than half their lengths, sometimes a single common time index is enough.} Then, precision is defined as the fraction of matched motifs among the discovered motifs, and recall as the fraction of matched motifs among the GT motifs. 

{This approach is not trivial to extend to multiple GT and discovered motif sets. In this case, a GT motif can be sufficiently similar to multiple discovered motifs, each belonging to a different set; and it is not trivial to decide to which one the GT motif should be attributed to. When the motifs have different lengths, this decision becomes even more complex. Our evaluation method PROM addresses these issues.}

\subsubsection{Used Datasets}\label{sec:related_work_used_datasets} 
Existing studies often use synthetic time series data to evaluate methods. These synthetic time series are constructed by inserting instances from a classification dataset into random-walk data, under the assumption that instances from the same class are similar in shape \citep{nunthanid_parameter-free_2012, correctness, setfinder}. As argued in the introduction (Section~\ref{sec:intro}), this approach leads to unrealistically easy TSMD tasks.
 
To evaluate TSMD methods on real instead of synthetic time series data, some studies use expert-annotated segmentation datasets \citep{van_leeuwen_normalization_2023, mueen_enumeration_nodate, mstamp}, because these are more common due to requiring less annotation effort. In this case, the GT labels consist of a partitioning of a time series into GT segments, where each segment corresponds to a certain ``state" of the process that generates the time series (e.g., a walking activity in a human activity recognition dataset). Such GT labels can be used to evaluate whether a discovered motif set is ``meaningful", by checking whether all its motifs are located within GT segments of the same state \citep{van_leeuwen_normalization_2023}. However, they cannot be used to directly evaluate TSMD, since the GT segments do not each correspond to an occurrence of a GT pattern.

Contrary to the datasets used in the literature, the benchmark time series we propose have GT motifs, and are constructed exclusively using real time series data, which arguably makes them more representative of realistic TSMD tasks.

\section{PROM: Precision-Recall under Optimal Matching}\label{sec:metric}
The metric we propose is called PROM: Precision-Recall under Optimal Matching.  Recall that we want to compare a set of discovered motif sets with a set of ground-truth motif sets.
PROM is calculated in three consecutive steps (see also Figure~\ref{fig:overview}): (1) discovered motifs are matched on an individual basis with GT motifs (Figure \ref{fig:uncolored_matching_segments}); (2) based on the result of step 1, discovered motif {\em sets} are matched optimally with GT motif sets (Figure \ref{fig:matching_segments}), which results in a {\em matching matrix} (Figure~\ref{fig:mm}); (3) notions of precision, recall, and derived measures such as F1 are defined based on this matching matrix (Figure~\ref{fig:pr}).   We next describe each of these steps in more detail.

\begin{figure}
\begin{subfigure}[t]{\linewidth}
    \centering
    \includegraphics[width=1\linewidth]{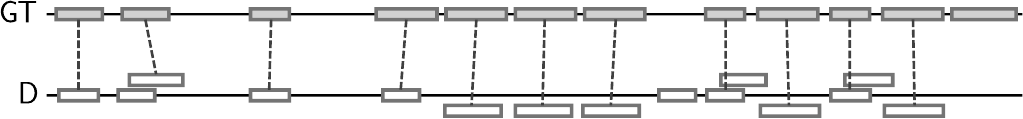}
    \caption{First, the discovered motifs are matched to the GT motifs based on their overlap in time}
    \label{fig:uncolored_matching_segments}
\end{subfigure}\hfill
\begin{subfigure}[t]{\linewidth}
    \centering
    \includegraphics[width=\linewidth]{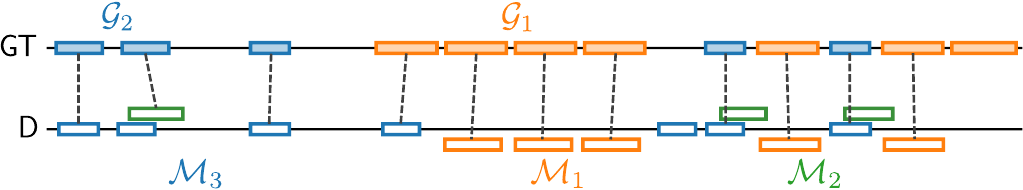}
    \caption{Second, the discovered motif sets are matched optimally with the GT motif sets based on the result of (a)}
    \label{fig:matching_segments}
\end{subfigure}\hfill
\begin{subfigure}[t]{0.5\linewidth}
    \centering
    \includegraphics[width=0.9\linewidth]{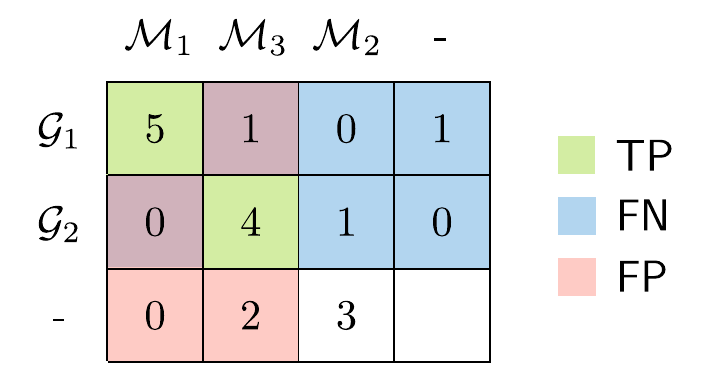}
    \caption{This results in a matching matrix $\mathbf{M}^{*}$}
    \label{fig:mm}
\end{subfigure}\hfill
\begin{subfigure}[t]{0.5\linewidth}
    \centering
    \includegraphics[width=0.85\linewidth]{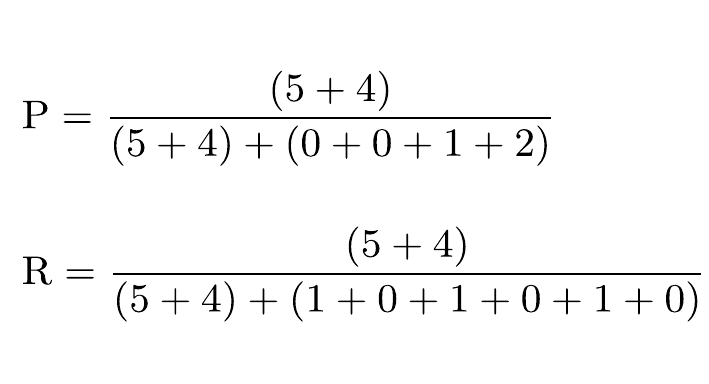}
    \caption{which can be used to calculate precision (P), recall (R), and F1-score} 
    \label{fig:pr}
\end{subfigure}\hfill
\caption{The evaluation process of PROM}
\label{fig:overview}
\end{figure}

\subsection{Matching individual motifs}
We define the {\em overlap rate} (OR) between two motifs $\alpha$ and $\beta$ as follows:
\begin{equation*}
    \text{OR}(\alpha, \beta) = \frac{|\alpha \cap \beta|}{|\alpha \cup \beta|}.
\end{equation*} 
We say that two segments $\alpha$ and $\beta$ are \emph{matchable} if $\text{OR}(\alpha, \beta) > 0.5$.  We then have the following property. \\

\begin{lemma}
Each discovered segment is matchable with at most one ground-truth segment.
\label{lemma:1}
\end{lemma}
\begin{proof}
By contradiction.  Assume a discovered segment $\alpha$ is matchable with two different ground-truth segments $\beta_1$ and $\beta_2$.  By definition, this implies
$\text{OR}(\alpha, \beta_i) = \frac{|\alpha \cap \beta_i|}{|\alpha \cup \beta_i|} > 0.5$, hence $|\alpha \cap \beta_i| > 0.5 |\alpha \cup \beta_i| \geq 0.5|\alpha|$ for $i=1,2$. So $\beta_1$ and $\beta_2$ both cover more than half of $\alpha$. This means they have to overlap, i.e., $|\beta_1\cap\beta_2| \neq \emptyset$, which contradicts the basic assumption that the ground-truth motifs do not overlap.
\end{proof}
Individual motifs are now matched as follows: for each GT segment $\beta$, find the discovered segment $\alpha$ with the highest $\text{OR}(\alpha, \beta)$. If it is greater than $0.5$, match $\beta$ with $\alpha$, otherwise, leave $\beta$ unmatched. Because of the above lemma, the result of this procedure does not depend on the order in which the ground-truth segments are processed: each $\alpha$ is matchable with at most one $\beta$, so one GT segment cannot ``steal" a candidate $\alpha$ from another GT segment, and the decision for one GT segment never interferes with the optimal decision for a later one. Further, since for each $\beta$ the most overlapping $\alpha$ is selected, the obtained matching maximizes the total $\text{OR}$ between the GT and the discovered segments. Note that since Lemma~\ref{lemma:1} holds for any OR-threshold between 0.5 and 1, a stricter OR-threshold can also be used.

\subsection{Matching motif sets, and the matching matrix}

Based on the obtained matching, we construct a contingency table $\textbf{M} \in \mathbb{N}^{\ngt\times\nd}$ in which cell $M_{i,j}$ indicates how often a segment from  GT motif set $\mathcal{G}_i$ was matched with a segment from discovered set $\mathcal{M}_j$.
The optimal matching between discovered motif sets and GT motif sets is then obtained by permuting the columns of $\textbf{M}$ so that the sum of the diagonal elements $M_{i,i}$ is maximized. This task reduces to a linear sum assignment problem and can be solved in $O(\max(g, d)^3)$ time using the Hungarian method \citep{Jonker1987ASA}. The resulting permutation $\pi$ is optimal in the sense that the motif sets are matched in such a way that the total number of matched segments inside them is maximal.

The matrix thus constructed is extended with one column containing for each GT motif set the number of segments not matched with any discovered segment, and a row containing for each discovered motif set the number of segments not matched with any GT motif.  This yields the {\em matching matrix} $\mathbf{M}^{*} \in \mathbb{N}^{(\ngt+1)\times(\nd+1)}$. See Figure \ref{fig:M_to_M_star} for an illustration. \\

\noindent \textbf{Properties of $\mathbf{M}^{*}$.} 
The sum of row $i$ of $\mathbf{M}^{*}$ is equal to the cardinality of $\mathcal{G}_{i}$: $\sum_{j=1}^{\nd+1}{M_{i, j}^{*}} = |\mathcal{G}_{i}|$ for $i=1, \dots, \ngt$. 
The sum of column $j$ is equal to the cardinality of $\mathcal{M}_{\pi(j)}$ for $j=1, \dots, \nd$. The bottom right entry $M^{*}_{\nd+1, \ngt+1}$ is undefined and is never used.

\begin{figure}
    \centering
    \begin{subfigure}{0.49\linewidth}
        \centering
        \includegraphics[scale=0.5]{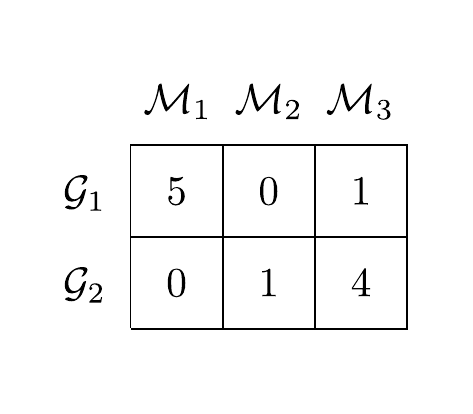}
        \caption{Contingency table $\mathbf{M}$}\label{fig:M}
    \end{subfigure}
    \begin{subfigure}{0.49\linewidth}
        \centering
        \includegraphics[scale=0.5]{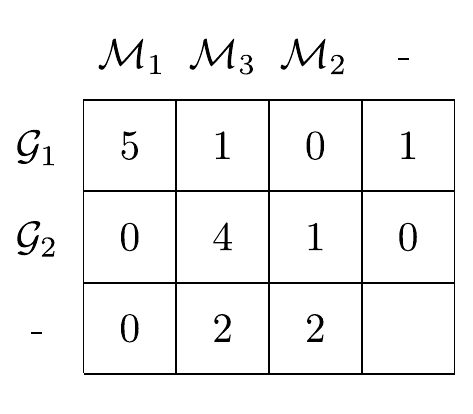}
        \caption{Matching matrix $\mathbf{M}^*$}\label{fig:M_star}
    \end{subfigure}
    \caption{The matching matrix $\mathbf{M}^{*}$ is obtained by permuting the columns of $\mathbf{M}$ with the optimal permutation $\pi$, and adding an extra column and row for unmatched GT and discovered motifs, respectively. In this specific example, $\pi(1) = 1, \pi(2) = 3$, and $\pi(3)=2$.}
    \label{fig:M_to_M_star}
\end{figure}

\subsection{Calculating Precision, Recall and F1-score}\label{sec:precrecf1}

The matching matrix defines $m=\min(d,g)$ matching pairs of motif sets $(\mathcal{G}_i, \mathcal{M}_{\pi(i)})$.  In each such pair, we define the number of \emph{true positives} $(\text{TP}_i)$, \emph{false negatives} $(\text{FN}_i)$ and \emph{false positives} $(\text{FP}_i)$ as follows (Figure \ref{fig:tpfpfn}):
\begin{itemize}
    \item $\text{TP}_i$ is the number of motifs in $\mathcal{G}_i$ for which a matching motif in $\mathcal{M}_{\pi(i)}$ exists. It is equal to the diagonal entry $M^{*}_{i, i}$.
    \item $\text{FN}_i$ is the number of motifs in $\mathcal{G}_i$ for which no matching motif in $\mathcal{M}_{\pi(i)}$ exists (either because they are matched with a segment in a different discovered set, or not at all).  It equals the sum of the off-diagonal entries in row $i$ of $\mathbf{M}^{*}$. 
    \item $\text{FP}_i$ is the number of motifs in $\mathcal{M}_{\pi(i)}$ not matched with any motif in $\mathcal{G}_i$. It equals the sum of the off-diagonal entries in column $i$ of $\mathbf{M}^{*}$.
\end{itemize}
\begin{figure}
\includegraphics[width=\linewidth]{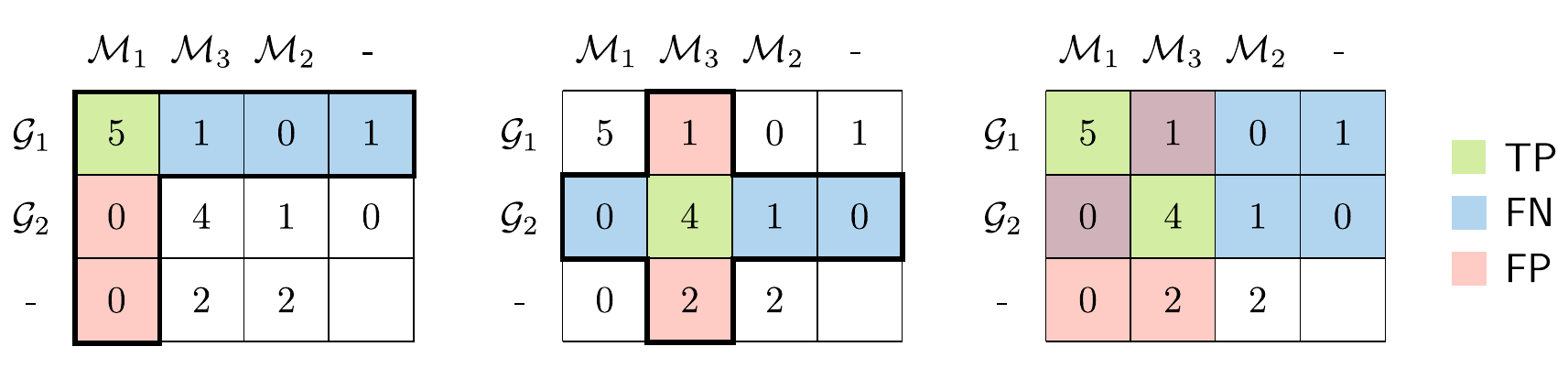}
\caption{An example matching matrix $\mathbf{M}^{*}$ when $\nd>\ngt$, and $\text{TP}_{i}$, $\text{FN}_{i}$, and $\text{FP}_{i}$ for $i=1$ (\emph{Left}) and $i=2$ (\emph{Center}). \emph{Right}: The total number of $\text{TP}$, $\text{FN}$, and $\text{FP}$.}
\label{fig:tpfpfn}
\end{figure}\quad
\begin{figure}
    \centering
    \includegraphics[scale=0.4]{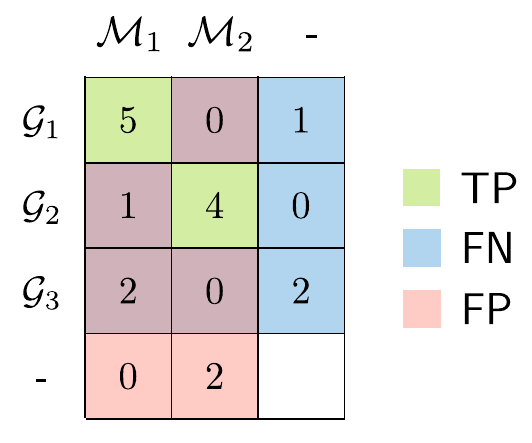}
    \caption{An example matching matrix $\mathbf{M}^{*}$ when $\ngt>\nd$; and the definitions of $\text{TP}$, $\text{FN}$, and $\text{FP}$. 
    }
    \label{fig:add_gt}
\end{figure}
When $g>d$, some GT motif sets are not matched with any discovered motif set. These have only false negatives: $\text{FN}_i = \sum^{\nd+1}_{j=1} M^{*}_{i, j} = k_i$ for $i=\nd+1, \dots, \ngt$. 

Conversely, when $d>g$, some discovered motif sets are not matched with any GT motif set. While this case can be treated symmetrically to $g>d$, it deserves special attention.
Since the GT by definition only includes the patterns of interest, it is possible that there are other patterns in the time series, that are just as real, but different from the GT patterns. Such an \emph{off-target} pattern may be not in the GT because it was unknown beforehand, or simply not of interest. Since these patterns might be conserved even better than the GT patterns, discovering them first should not necessarily be penalized \citep{minnen, balasubramanian}. A typical example of an off-target pattern is a recurring calibration signal in an electrocardiogram: while the heartbeats are of interest, the calibration signal might be conserved better and may be discovered first by TSMD methods \citep{guiding}. 
For this reason, a user might choose to ignore false positive discoveries on the level of motif sets.  This boils down to defining $\text{FP}_j = 0$ for $j=\ngt+1, \dots, \nd$, rather than counting all motifs in the false positive motif sets: $\text{FP}_j = \sum_{i=1}^{\ngt+1} M^{*}_{i,j}$.

The above $\mathrm{FP}_i$, $\mathrm{TP}_i$ and $\mathrm{FN}_i$ are summed up:
\begin{equation}
    \text{TP} = \sum_{i=1}^{\nmin}{\text{TP}_{i}}, 
    \qquad
    \text{FN} = \sum_{i=1}^{\ngt}{\text{FN}_{i}},\ \qquad
    \text{FP} = \sum_{j=1}^{\nd}{\text{FP}_{j}} 
    \label{eq:fp_alt}
\end{equation} 
and precision, recall, and F1-score are defined based on the sums:
\begin{equation}
    \text{P} = \frac{\text{TP}}{\text{TP} + \text{FP}},\qquad
    \text{R} = \frac{\text{TP}}{\text{TP} + \text{FN}}, \qquad 
    \text{F} = 2\frac{ \mathrm{P} \cdot \mathrm{R}}{\mathrm{P} + \mathrm{R}}. 
\end{equation}
We define $\mathrm{F}=0$ when $\mathrm{TP}=0$. Note that these formulas correspond to micro-averaging. We could also use macro-averaging, where P, R and F are calculated separately for each motif set and then averaged. However, in this case, unmatched motif sets require special care, as F is undefined for them. Appendix \ref{appendix:macro-averaging} details macro-averaged PROM.

\section{Proposed Benchmark}\label{sec:benchmarks}
This section introduces TSMD-Bench, our benchmark for TSMD. It consists of 14 benchmark datasets, each consisting of multiple time series constructed from a time series classification dataset. We first describe how we create a benchmark dataset from a classification dataset, and then how we selected suitable classification datasets. 

\subsection{Constructing a TSMD Dataset from a Classification Dataset}\label{sec:benchmarks_construction}

\begin{figure}
    \centering
    \includegraphics[width=\linewidth]{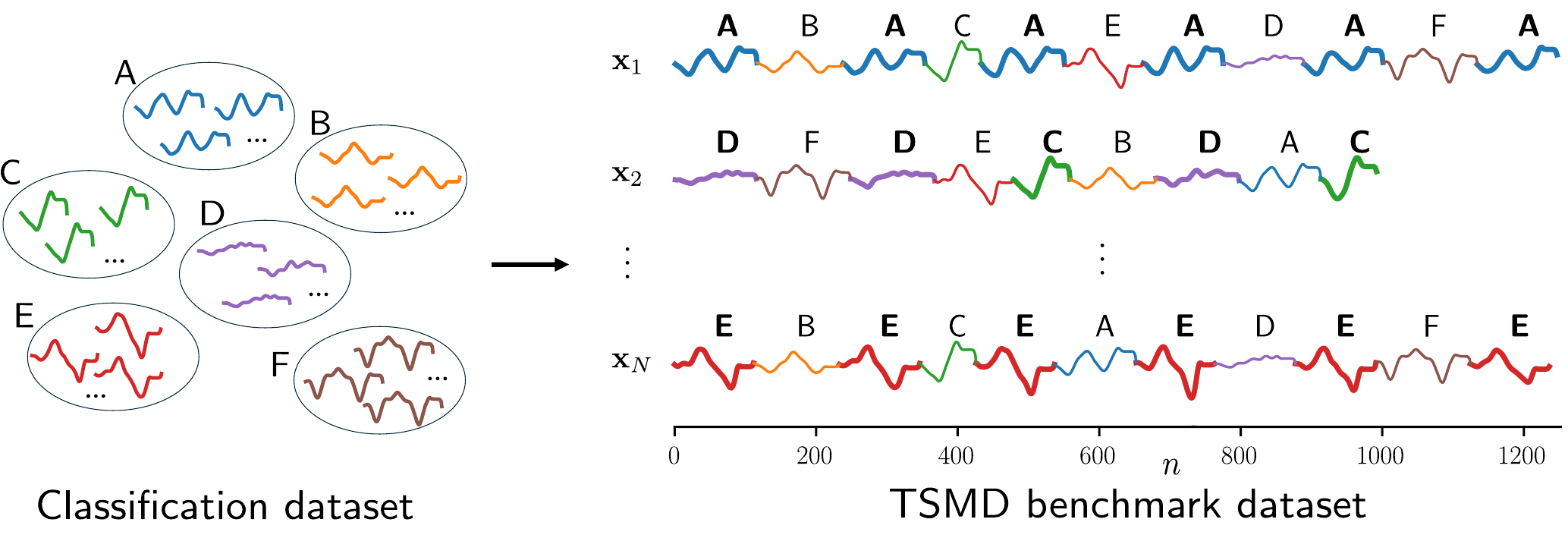}
    \caption{From a classification dataset in which the instances of the same class are similar in shape, we construct a TSMD benchmark consisting of $N$ benchmark time series, for which the GT motif sets (indicated in bold) are available.}
    \label{fig:benchmark-generation}
\end{figure}

Each classification dataset contains multiple relatively short, labeled time series instances. 
We only use classification datasets in which the instances of the same class are similar to each other, and dissimilar to those in other classes (such as the one shown in Figure~\ref{fig:benchmark-generation}). 

Each benchmark time series is a concatenation of classification instances. The instances that belong to the same class constitute a GT motif set. To minimize the number of \emph{off-target} patterns (defined in Section \ref{sec:precrecf1}) in a benchmark time series, we impose a certain structure on them. For example, consider a classification dataset with six classes: \textsf{A}, \textsf{B}, \textsf{C}, \textsf{D}, \textsf{E} and \textsf{F}. If we would concatenate the instances of these classes in a random order, a time series such as \textsf{\textbf{\color{tabblue}AB}CF\textbf{\color{tabblue}AB}DE} could occur, where the pattern \textsf{AB} is additional to the GT patterns \textsf{A} and \textsf{B}. To solve this, we impose a certain structure on each time series we construct: between any two GT motifs, we place a single instance from a class that does not repeat, e.g., \textsf{\textbf{\color{tabblue}A}C\textbf{\color{tabblue}A}F\textbf{\color{taborange}B}D\textbf{\color{tabblue}A}E\textbf{\color{taborange}B}} (Figure~\ref{fig:benchmark-generation}). Note that the imposed structure does not prevent all off-target patterns. For example, an instance of class \textsf{A} could consist of two consecutive parts \textsf{X} and \textsf{Y}, and an instance of class \textsf{B} of \textsf{X} and \textsf{Z}. In this case, any benchmark time series that contains both class \textsf{A} and \textsf{B} contains the pattern \textsf{X}, which is off-target in this case.

To generate a time series with this structure that has $\ngt$ GT motif sets, at least $3\ngt-1$ classes are required. $\ngt$ classes are needed for the GT motif sets. Every GT motif set has at least two motifs, so there are at least $2\ngt$ motifs to be interleaved with an instance of a class that does not repeat, which requires $2\ngt- 1$ different classes. Thus, we need at least $3\ngt-1$ classes in total. For a classification dataset having $c$ classes, the number of GT motif sets is therefore bounded by $\ngt_{\max} = \lfloor \frac{c+1}{3} \rfloor$.

For every benchmark time series, we sample the number of GT motif sets $\ngt$ uniformly from $[1:\ngt_{\max}]$; randomly select the $\ngt$ classes that are used to construct the GT motif sets; and randomly select the instances that are concatenated. \\

\noindent\textbf{Validation and test set.} For each benchmark dataset, we construct a validation set, which can be used to tune the hyperparameters of TSMD methods, and a test set, to evaluate the performance of TSMD methods. To ensure that no classification instance is present in both a validation and test time series, we first split the classification dataset into a validation and a test set, and then use these sets to separately construct the validation and test set of the TSMD benchmark dataset (both use the same validation-test split). Hence, the scheme shown in Figure \ref{fig:benchmark-generation} is applied twice, once for the validation set ($N=50$) and once for the test set ($N=200$).

\subsection{Selecting Suitable Classification Datasets}\label{sec:dataset_selection}

\noindent To construct TSMD-Bench, we make use of the classification datasets from the well-known UCR and UAE Time Series Classification Archive \citep{UCRArchive2018}. 

We select classification datasets based on two criteria. First, we only select the datasets with five or more classes ($c \geq 5$), such that at least two GT motif sets can be generated per benchmark time series ($\ngt_{\max} \geq 2$). Second, as our benchmark construction assumes that the instances of the same class are similar in shape (Section \ref{sec:benchmarks_construction}), we only select the classification datasets whose classes are distinguishable in an unsupervised manner (i.e., easily clusterable). A classification dataset can be categorized based on whether is is fixed- or variable-length and whether it is uni- or multivariate. For the fixed-length, univariate category we use the results of a recent clustering benchmark paper \citep{javed_benchmark_2020}, and select the datasets for which the Adjusted Rand Index (ARI) averaged over multiple distance measures and clustering algorithms is higher than 0.5. Five datasets adhere to this criterion: \textsf{ECG5000}, \textsf{Fungi}, \textsf{Plane}, \textsf{Mallat}, and \textsf{Symbols}. The datasets of the other three categories, we cluster ourselves, using the $k$-medoids algorithm ($k=c$) with Euclidean distance (ED) and Dynamic Time Warping (DTW) distance.\footnote{Strictly speaking, DTW distance is not a distance measure as it does not guarantee the triangle inequality to hold.}\footnote{To apply ED on a variable-length dataset, we resample the instances such that their length is equal to the average length in the dataset.} Figure \ref{fig:clustering_result} shows the obtained ARI values per dataset, and indicates which classification datasets are deemed suitable for our TSMD benchmark (apart from the five datasets mentioned above).

\begin{figure}
    \centering
    \includegraphics[width=\linewidth]{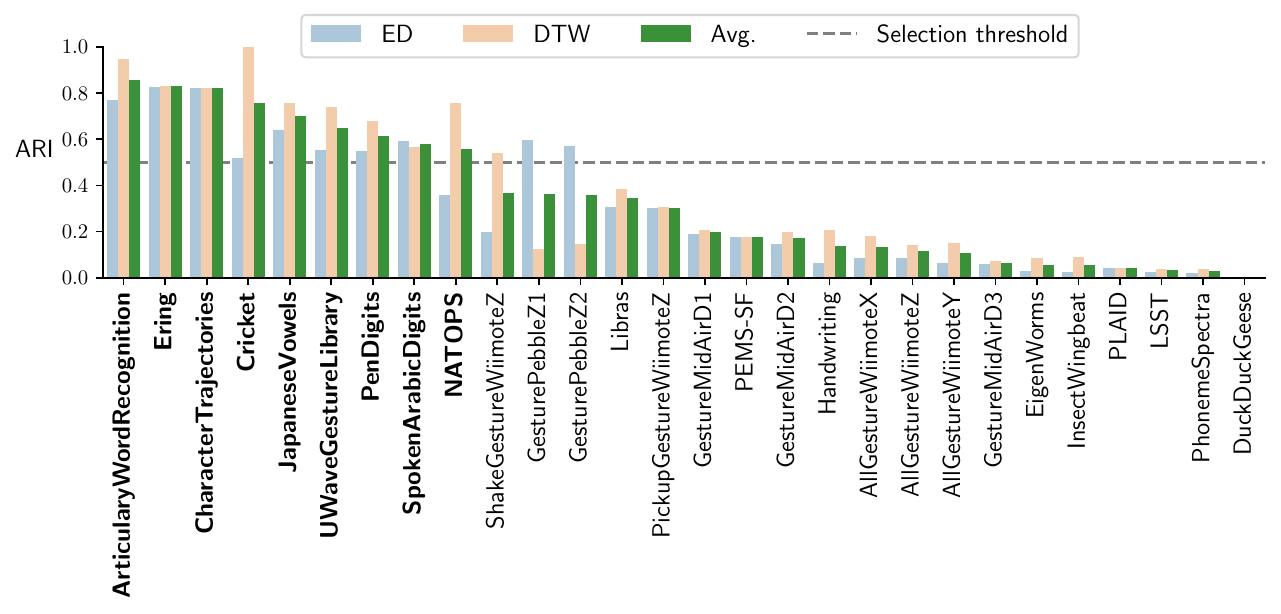}
    \caption{The obtained ARI values by the $k$-medoids algorithm for a subset of UCR and UEA Time Series Classification datasets (more than 5 classes; multivariate, variable-length, or both), using ED and DTW. The datasets we select to construct a TSMD benchmark dataset (average ARI $\geq$ 0.5) are shown in boldface.}
    \label{fig:clustering_result}
\end{figure}

Table \ref{tab:benchmark_overview} shows an overview of our entire TSMD benchmark. It consists of 14 benchmark datasets in total: 11 have fixed-length motifs, 5 of which are univariate and 6 multivariate; the other 3 include variable-length motifs and are multivariate.

\begin{table}[]
\centering
\begin{tabular}{p{4cm}>{\raggedleft\arraybackslash}p{1cm}>{\raggedleft\arraybackslash}p{1cm}>{\raggedleft\arraybackslash}p{1cm}>{\raggedleft\arraybackslash}p{1cm}>{\raggedleft\arraybackslash}p{1cm}>{\raggedleft\arraybackslash}p{1cm}}
\toprule
                 Benchmark Dataset & $c$ & $l_{\min}$ & $l_{\max}$ & $D$ & $\ngt_{\max}$ & $n_{\text{avg}}$  \\ \midrule
                  \textsf{ECG5000} &   5 &  140       &    140     &  1  & 2             &   980             \\ 
                    \textsf{Fungi} &  18 &  201       &    201     &  1  & 6             &  4639             \\

                   \textsf{Mallat} &   8 & 1024       &   1024     &  1  & 3             & 12370             \\
                    \textsf{Plane} &   7 &  144       &    144     &  1  & 2             &  1584             \\ 
                  \textsf{Symbols} &   6 &  398       &    398     &  1  & 2             &  3582             \\ \midrule
\textsf{ArticularyWordRecognition} &  25 &  144       &    144     &  9  & 8             &  3387             \\
                  \textsf{Cricket} &  12 & 1197       &   1197     &  6  & 4             & 14201             \\
                    \textsf{ERing} &   6 &   65       &     65     &  4  & 2             &   585             \\
                   \textsf{NATOPS} &   6 &   51       &     51     & 24  & 2             &   459             \\
                \textsf{PenDigits} &  10 &    8       &      8     &  2  & 3             &   129             \\ 
      \textsf{UWaveGestureLibrary} &  8  &  315	    &    315	 &  3  & 3             &  4077\\ \midrule
    \textsf{CharacterTrajectories} &  20 &   60       &    182     &  3  & 7             &  2759             \\
           \textsf{JapaneseVowels} &   9 &    7       &     29     & 12  & 3             &   166             \\
       \textsf{SpokenArabicDigits} &  10 &    4       &     93     & 13  & 3             &   523             \\
\bottomrule
\end{tabular}
\caption{Overview of the proposed TSMD benchmark. For each of the 14 benchmark datasets, the table includes properties of the respective classification dataset: the number of classes $c$, and min.\ and max.\ length of the instances, $l_{\min}$ and $l_{\max}$, the number of dimensions $D$; as well as the max.\ number of GT motif sets $\ngt_{\max}$ (=$\lfloor \frac{c+1}{3} \rfloor$), and average time series length $n_{\text{avg}}$.}
\label{tab:benchmark_overview}
\end{table}

\section{Experiments}
\noindent Our experiments evaluate the usefulness of PROM and TSMD-Bench. They address three research questions:
\begin{enumerate}[label=\textbf{RQ\arabic*}:, leftmargin=\widthof{[RQ3:]}+\labelsep]
    \item How does PROM compare to the existing metrics? 
    \item What does PROM tell us about the performance of existing TSMD methods?
    \item What is the added value of TSMD-Bench over benchmarks based on random walk?
\end{enumerate}

To answer \textbf{RQ1}, we generate 30 800 TSMD results (($\mathcal{G}, \mathcal{M}$)-pairs) by applying 11 TSMD methods to TSMD-Bench (14 datasets, 200 test time series per dataset). We then investigate how similarly the considered metrics rank the obtained results, and have a closer look at the cases in which the metrics disagree. For \textbf{RQ2}, we reuse the results generated for \textbf{RQ1}, and investigate the PROM (P, R, and F) values obtained by the considered methods. \textbf{RQ3} addresses the claim we have made in Section~\ref{sec:intro} about TSMD benchmarks that use random walk.

\subsection{Comparing PROM to the existing metrics (RQ1)}
\subsubsection{Considered Metrics}\label{sec:considered_metrics}
We consider only the metrics in Table \ref{tab:overview_metrics} that evaluate multiple motif sets of any cardinality: Apart from PROM, we consider the overlap-based metric proposed by \citet{correctness}, called \emph{correctness}, and the \emph{score} metric by \citet{setfinder}. \\

\noindent \textbf{Correctness (C).}  Given GT motif sets $\mathcal{G}$, discovered motif sets $\mathcal{M}$, and a matching $\pi$ that matches $\mathcal{G}_i$ to $\mathcal{M}_{\pi(i)}$ for $i = 1, \dots, m$, where $m = \min(g, d)$, \emph{correctness} is defined as:
\begin{equation}
\text{C}(\mathcal{G}, \mathcal{M}) = \frac{1}{m} \sum_{i = 1}^m \frac{\sum  \{ \text{OR}(\alpha, \beta) \ | \ \beta \in \mathcal{G}_i, \alpha \in \mathcal{M}_{\pi(i)}, \ \text{OR}(\alpha, \beta) > 0.5\}}{|\mathcal{G}_i|}.
\label{eq:correctness}
\end{equation}
Thus, for each match $(\mathcal{G}_i, \mathcal{M}_{\pi(i)})$, the metric considers the overlap rates between each $\beta \in \mathcal{G}_i$ and $\alpha \in \mathcal{M}_{\pi(i)}$ that exceed 0.5, sums them up, and normalizes by $|\mathcal{G}_i|$. The C value is calculated as the average over all $m$ matches. The way $\pi$ is obtained was not described in the original paper; hence, we assume that the metric uses the $\pi$ that maximizes the final correctness value.

This metric has certain shortcomings. First, as mentioned in the Related Work (Section~\ref{sec:rw}), C does not penalize falsely discovered motifs. Second, unlike PROM, C does not assign a unique discovered motif to every GT motif.
If a $\beta \in \mathcal{G}_i$ has multiple $\alpha$'s in $\mathcal{M}_{\pi(i)}$ for which $\mathrm{OR}(\alpha, \beta) > 0.5$, their $\mathrm{OR}$ values are simply accumulated. This rewards detecting the same GT segment multiple times, which is undesirable. It can also cause the value of C to exceed 1. Third, C does not penalize unmatched GT motif sets when $g > d$. Only this shortcoming can be easily alleviated, by defining the correctness of an unmatched GT motif set to be 0. This corresponds to replacing $m$ with $g$ in Eq.~\ref{eq:correctness}. In the remainder of this paper, we use C with and without this modification, and denote them as C$_{g}$ and C$_{m}$, respectively.\\

\noindent \textbf{Score (S).} The \emph{score} metric considers each possible ``two-level'' matching, which is a matching first between the GT and discovered motif sets, and then between the motifs of each pair of matched motif sets. The score value for such a two-level matching is equal to the sum of the differences in start indices between matched motifs and the lengths of the unmatched motifs (either GT or discovered). S obtains the lowest obtainable  value across all possible two-level matchings. A lower score indicates a better result.

The authors do not specify how unmatched motif sets affect S. A reasonable approach is to penalize the unmatched GT motif sets by adding the lengths of their motifs to the score value, and let the user choose whether the same should be done for unmatched discovered motif sets (as we did with PROM, see Section~\ref{sec:precrecf1}). In the remainder of the paper, we use this version of S.

The score metric has three main drawbacks. First, it is not suitable to evaluate variable-length motifs: For a certain GT motif $\beta$, a discovered motif $\alpha$ that starts at the same index is deemed a perfect match, even if $\mathrm{OR}(\alpha, \beta) \approx 0$ due to $\alpha$'s length being completely off. Second, S is highly dependent on the length and the number of GT motifs, which makes it difficult to compare values across datasets. A third drawback is its time complexity: in an illustrative case where $g=d=m$, and each motif set has $k$ motifs, the time complexity of S is $O(m^2k^3+m^3)$, which scales poorly when the number of motifs scales linearly with the length of the time series $n$. PROM's time complexity is $O(km + m^3)$ in this illustrative case, and would remain linear in $n$.

\subsubsection{Relationship with P, R and F}

\begin{figure}
\begin{subfigure}[t]{0.5\linewidth}
    \centering
    \includegraphics[width=0.75\linewidth]{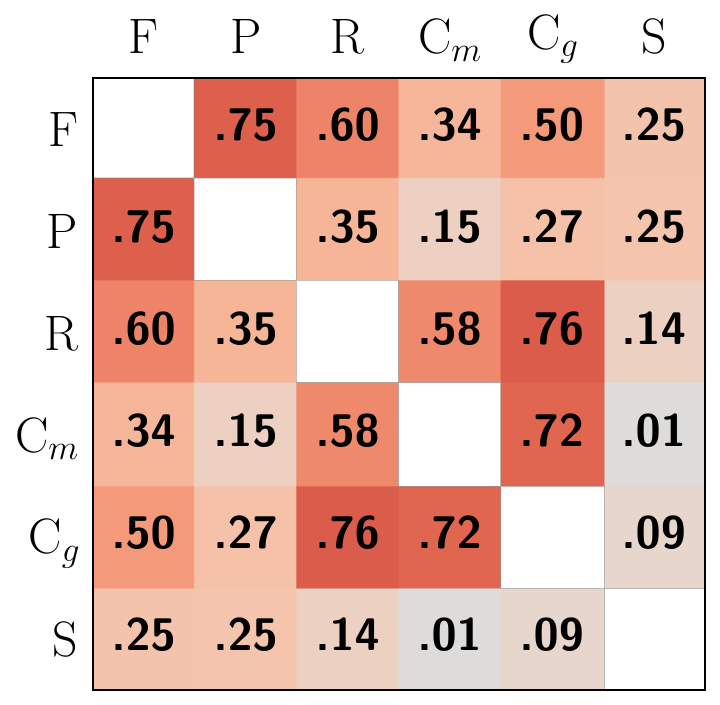}
    \caption{}
    \label{fig:tau_all}
\end{subfigure}
\begin{subfigure}[t]{0.5\linewidth}
    \centering
    \includegraphics[width=0.75\linewidth]{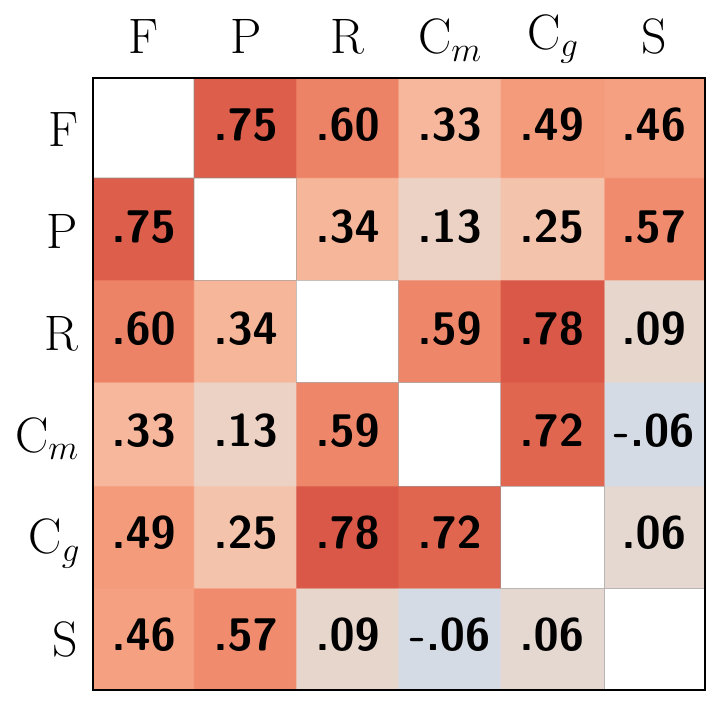}
    \caption{} 
    \label{fig:tau_avg}
\end{subfigure}
\caption{The Kendall rank correlation coefficient $\tau$ between the rankings obtained by the considered evaluation metrics. We calculate $\tau$ over the results of all benchmark datasets combined (a), and the average $\overline{\tau}$ when calculating $\tau$ for each benchmark dataset separately (b). Note that PROM and S were configured to not penalize off-target discovered motif sets.}\label{fig:taus}
\end{figure}

\noindent To assess the relationship between two evaluation metrics, we examine how similarly they rank the obtained results. To this end, we let the metrics rank the results from best to worst, and then calculate the Kendall rank correlation coefficient $\tau$ between their rankings. This coefficient is equal to $1$ when the two metrics rank the results exactly the same, and equal to $-1$ when one ranking is the reverse of the other. We calculate $\tau$ over the results of all benchmark datasets combined, as well as the average $\overline{\tau}$ after calculating $\tau$ separately for each of the 14 benchmark datasets. The latter is useful since the range of S considerably differs across benchmark datasets, and consequently, its ranking of the combined results is not meaningful. Figure \ref{fig:taus} shows the obtained values for each pair of considered evaluation metrics (F, P, R, $\mathrm{C}_m$, $\mathrm{C}_g$, and S).

Since both versions of the correctness metric ($\mathrm{C}_g$ and $\mathrm{C}_m$) do not penalize falsely discovered motifs, they rank the results most similarly to R, rather than F or P (Figure \ref{fig:tau_all}). The example in Figure \ref{fig:example_c_1} illustrates this: both GT motif sets are discovered imprecisely ($\mathrm{P}=0.35$ and $\mathrm{F}=0.52$), but nevertheless both $\mathrm{C}_g$ and $\mathrm{C}_m$ assign a relatively high value ($= 0.92$), similar to the value of R ($=1.00$). As $\mathrm{R}$ and $\mathrm{C}_g$ both measure to what extent the GT motifs are discovered, they produce similar values in most cases. Exceptions occur when the overlap rates between the matched GT and discovered motifs are only slightly higher than 0.5 (Figure \ref{fig:example_c_2}), multiple discovered motifs match to the same GT motif (Figure \ref{fig:example_c_3}), or the metrics obtain a different matching between the GT and discovered motif sets.
$\mathrm{C}_m$ and R differ in more cases as $\mathrm{C}_m$ does not penalize unmatched GT motif sets.
\begin{figure}
\begin{subfigure}[t]{\linewidth}
    \centering
    \includegraphics[width=0.75\linewidth]{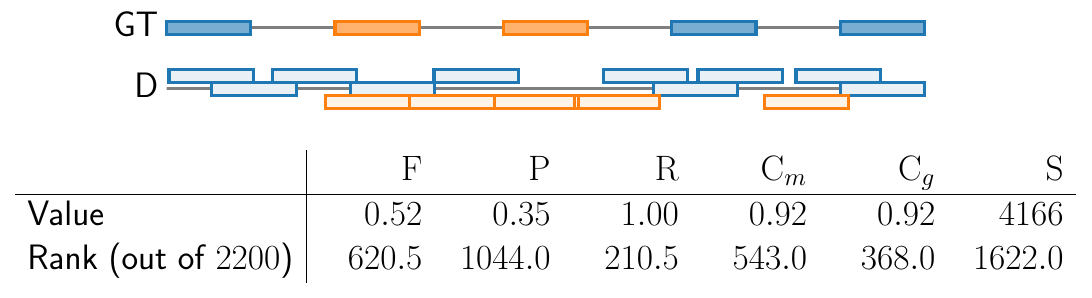}
    \caption{}\label{fig:example_c_1}
\end{subfigure}
\begin{subfigure}[t]{\linewidth}
    \centering
    \includegraphics[width=0.75\linewidth]{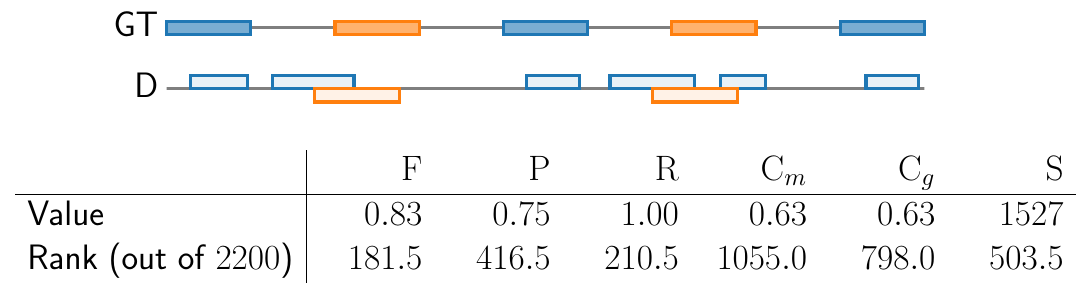}
    \caption{}\label{fig:example_c_2}
\end{subfigure}
\begin{subfigure}[t]{\linewidth}
    \centering
    \includegraphics[width=0.75\linewidth]{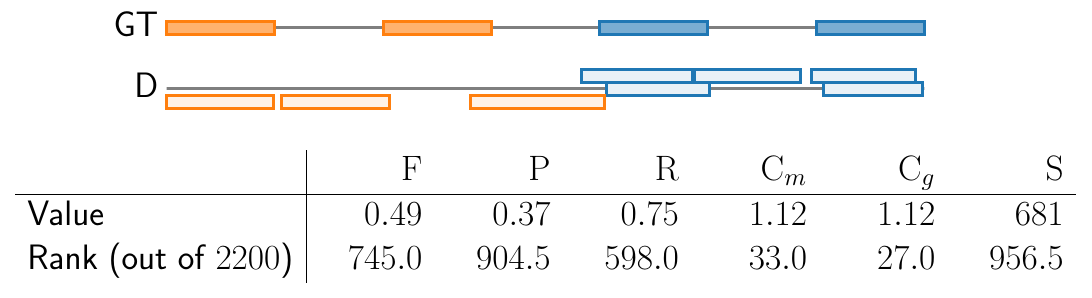}
    \caption{}\label{fig:example_c_3}
\end{subfigure}
\caption{Examples that demonstrate the relationships between PROM (P, R, and F) and the correctness metric ($\mathrm{C}_{m}$ and $\mathrm{C}_{g}$): (a) many false-positive motifs (b) overlap rates only slightly higher than 0.5  (c) multiple discovered motifs match to the same GT motif. \textsf{Value} denotes the performance value that a metric assigns to the result. \textsf{Rank} denotes the rank the metric assigns to the result in the respective benchmark dataset.}\label{fig:examples_disagreement_c}
\end{figure}

For S, only the average $\overline{\tau}$ (Figure \ref{fig:tau_avg}) is relevant, since its ranking of all results combined obtained is not meaningful. As S penalizes both falsely discovered motifs and missed GT motifs, it is expected to rank the results similarly to F.  Somewhat surprisingly, its ranking corresponds only moderately to the ranking of F ($\overline{\tau}=0.46$), and more to the one of P ($\overline{\tau}=0.57$). 
{A possible explanation is that the overall impact of falsely discovered motifs on S can potentially dominate that of missed GT motifs, as a motif of either type contributes equally to the S value, but the former type can be more numerous.} The rank assigned by S is therefore mainly determined by the number of falsely discovered motifs (just like P), rather than the missed GT motifs (just like R). Figure~\ref{fig:example_s_1} illustrates this: even though most of the GT motifs are missed, S and P both rank this result relatively high ($473^{\mathrm{th}}$ and $450^{\mathrm{th}}$, respectively), while F ranks the result low ($1703^{\mathrm{th}}$). The evaluation by S also differs from F and P when there are considerable length differences between the GT and discovered motifs. This is illustrated in Figure~\ref{fig:example_s_2}: PROM does not match the discovered motifs to the GT motifs, resulting in $\mathrm{F}=\mathrm{P}=0$, while S ranks the result high, since it only measures differences between the start points of the motifs. 

We can now answer \textbf{RQ1}. The existing metrics do not evaluate TSMD in all aspects: the correctness metric ($\mathrm{C}_g$ and $\mathrm{C}_m$) only measures the extent to which the GT motifs are discovered (recall), and the score metric (S) mainly measures precision. In contrast, PROM is generally applicable, and provides a complete evaluation by independently measuring recall and precision.

\begin{figure}
\begin{subfigure}[t]{\linewidth}
    \centering
    \includegraphics[width=0.75\linewidth]{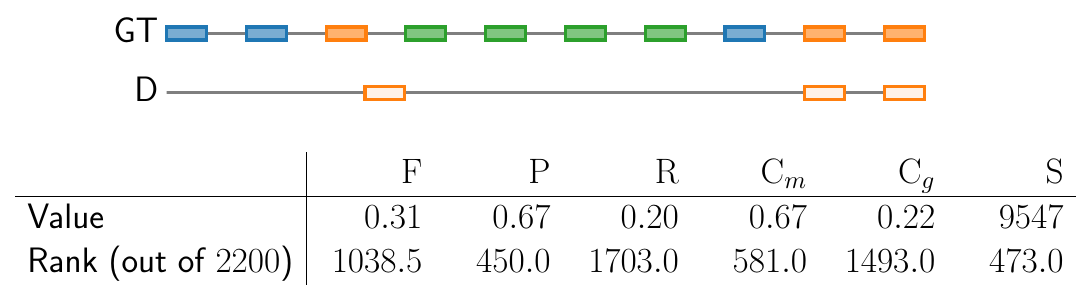}
    \caption{}\label{fig:example_s_1}
\end{subfigure}
\begin{subfigure}[t]{\linewidth}
    \centering
    \includegraphics[width=0.75\linewidth]{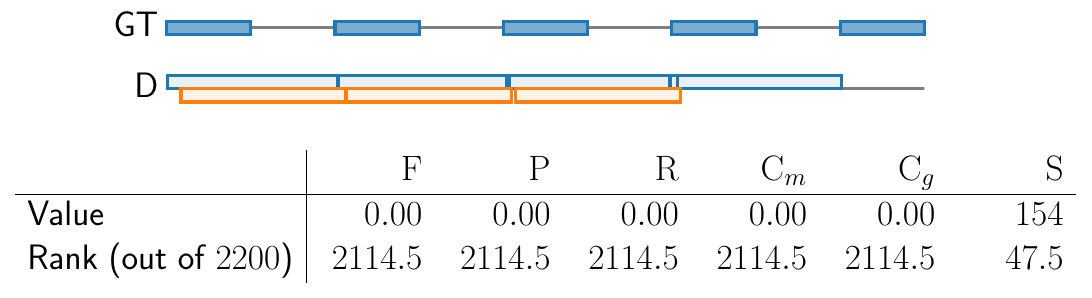}
    \caption{}\label{fig:example_s_2}
\end{subfigure}
\caption{Example results that are ranked high by S, but not by PROM: (a) many missed GT motifs but no false positive motifs, and (b) discovered motifs with similar start indices as the GT motifs but with completely different lengths.}\label{fig:examples_disagreement_s}
\end{figure}

\subsection{Performance values per benchmark dataset (RQ2)}

\subsubsection{Applied methods}\label{sec:methods}
We applied the 11 TSMD methods listed in Table \ref{tab:overview_methods}.  Below, we briefly discuss the properties we use to categorize these methods.\\

\noindent \textbf{Fixed- or variable-length:} whether a method finds only motifs of identical length or not. If not, we say a method finds \emph{variable-length} motifs. We further distinguish \emph{inter}-variable-length motifs, whose lengths can differ across motif sets but not within the same motif set, and \emph{intra}-variable-length motifs, whose lengths can differ even within the same motif set. Note that the variable-length motifs in our benchmark datasets are intra- and not inter-variable-length motifs. \\

\noindent \textbf{Uni- or multivariate:} whether a method is only applicable to univariate time series (time series with a single dimension), or also to time series with multiple dimensions. \\

\noindent \textbf{Symbolic approximation or not:} whether a method first converts the input time series to a symbolic sequence using, e.g., SAX \citep{emma}. Methods that do so are said to be less precise \citep{motiflets}, because the approximations made by the initial conversion step may lead to two identical symbolic subsequences whose corresponding time series subsequences are not alike. 

\begin{table}
\newcommand{\rot}[1]{\rotatebox{90}{#1}}
\centering
\begin{tabular}{p{5cm}p{1.25cm}p{0.125cm}p{0.125cm}p{4.5cm}}
Method                                               & \rot{Variable-length} & \rot{Appl. to multivar.} & \rot{No symbolic approx.} & Hyperparameters \\ \midrule
\textsc{EMMA} \citep{emma}                           & \xmark                                 & \xmark                 & \cmark\tablefootnote{EMMA converts the time series to a symbolic sequence but only for efficiency; that is, it still compares the corresponding time series subsequences.} & $l, r$, $s\in \{4, \dots, \frac{l}{2} \}$, $a \in \{ 2, 4, 8, 16 \}$ \\
\textsc{MrMotif} \citep{mrmotif}                     & \xmark                                 & \xmark                 & \xmark & $l$                                                                                                                              \\
\textsc{SetFinder} \citep{setfinder}                 & \xmark                                 & \xmark                 & \cmark & $l$, $r$                                                                                                                         \\ 
\textsc{LatentMotifs} \citep{grabocka_latent_2017}   & \xmark                                 & \xmark                 & \cmark & $l$, $r$                                                                                                                         \\
\textsc{GV-Sequitur} \citep{grammarviz}              & \cmark \ (intra)                       & \xmark                 & \xmark & $w \in \{8, 16, \dots, 2^{\lfloor \log_{2}{2l_{\max}}\rceil} \}$, $s \in \{ 4, \dots, \frac{w}{2}\}$, $a \in \{ 2, 4, 8, 16\}$.  \\
\textsc{GV-RePair} \citep{grammarviz}                & \cmark \ (intra)                       & \xmark                 & \xmark & Same as \textsc{GV-Sequitur}                                                                                                     \\ 
\textsc{VALMOD} \citep{linardi_matrix_2018}          & \cmark \ (inter)                       & \xmark                 & \cmark & $l_{\min}$, $l_{\max}$, $r_f \in \{2, 3, \dots, 6\}$                                                                             \\
\textsc{STUMPY mmotifs} \citep{stumpy}               & \xmark                                 & \cmark\tablefootnote{Although \textsc{MMotifs} can find \emph{sub-dimensional} motifs (motifs that span a subset of the dimensions), we configure it to only find motifs that span all dimensions as our benchmarks are constructed for this setting.} & \cmark & $l$, $r$ \\
\textsc{VONSEM} \citep{vonsem}                       & \cmark \ (inter)                       & \xmark                 & \cmark & $l_{\min}$, $r$                                                                                                                  \\
\textsc{Motiflets} \citep{motiflets}                 & \xmark                                 & \xmark                 & \cmark & $l_{\min}$, $l_{\max}$, $k_{\max}$                                                                                               \\
\textsc{LoCoMotif} \citep{locomotif}                 & \cmark \ (intra)                       & \cmark                 & \cmark & $l_{\min}$, $l_{\max}$, $\rho \in \{ 0.1, 0.2, \dots, 0.9\}$, $\texttt{warp} \in \{ \text{True}, \text{False} \}$                          \\
\bottomrule
\end{tabular}
\caption{The applied methods, their properties, and their relevant hyperparameters.}\label{tab:overview_methods}
\label{tab:methods_overview}
\end{table}

\subsubsection{Hyperparameters}
For each hyperparameter, we select one value for the whole test set (instead of one value per time series). Each method was allowed to discover $\ngt_{\max}$ (Section \ref{sec:dataset_selection}) motif sets (see Table~\ref{tab:benchmark_overview} for the value per benchmark dataset). The hyperparameters related to overlap are set to their default value. Some hyperparameters are data-dependent and are set directly based on the validation set of either the benchmark dataset or the respective classification dataset. $l$, $l_{\min}$ and $\l_{\max}$ represent the fixed, min. and max. motif length. For fixed-length datasets, they are set to the actual length of the classification instances (Table~\ref{tab:benchmark_overview}). For variable-length datasets, $l$ is set to the mean length, and $l_{\min}$ and $\l_{\max}$ respectively to the $0.1$- and $0.9$-quantile of the lengths.\footnote{To be insensitive to outliers, e.g., $l_{\min}=4$ for the \textsf{SpokenArabicDigits} dataset, see Table~\ref{tab:benchmark_overview}} $k_{\max}$ (maximum motif set cardinality) of \textsc{Motiflets} is set to twice its actual value.\footnote{Because Motiflets uses the elbow method to find suitable values for $k$, we set its value higher than its actual value} $r$ represents the maximum Euclidean distance (ED) between a central motif and any another motif in a fixed-length discovered motif set. It is set as follows: For each class in the classification dataset, we find the instance that minimizes the maximum ED to all other instances in that class (the most ``central'' instance). Then, we set $r$ to the maximum over all classes.\footnote{For variable-length datasets, we first resample the instances to the mean length such that ED can be calculated.} The rest of the hyperparameters (window size $w$, word size $s$, and alphabet size $a$ for the methods that use a symbolic representation; radius factor $r_f$ of \textsc{VALMOD}; strictness $\rho$ and \texttt{warp} of \textsc{LoCoMotif}) have a predefined set of possible values (Table \ref{tab:methods_overview}) and are set to maximize the F1-score on the validation set using exhaustive grid search. For methods only applicable to univariate time series, the best dimension of the multivariate benchmarks to use is also determined through the grid search.

\subsubsection{Results}
\begin{figure}
\begin{subfigure}[t]{\linewidth}\centering
    \includegraphics[width=\linewidth]{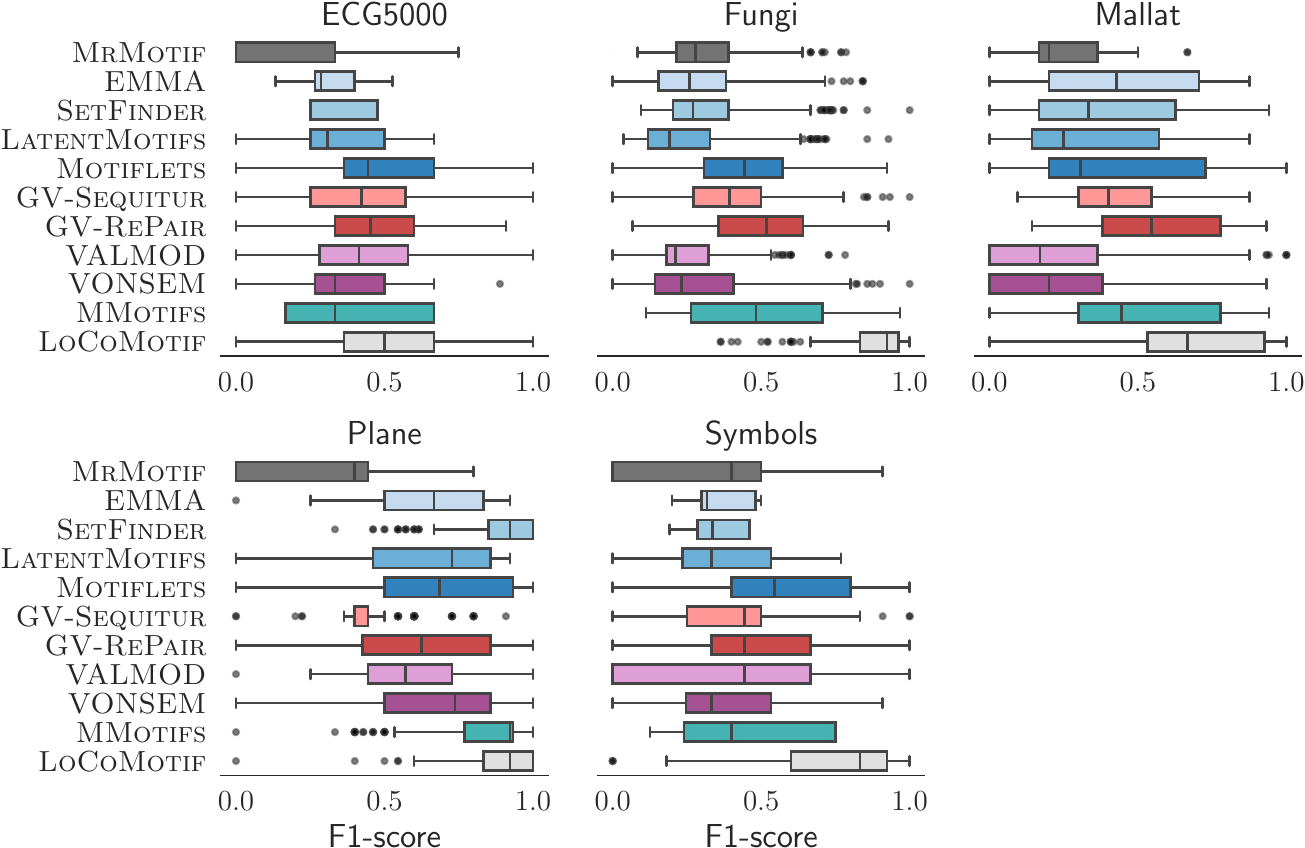}
\caption{}\label{fig:boxplots-a}
\end{subfigure}
\begin{subfigure}[t]{\linewidth}\centering
    \includegraphics[width=\linewidth]{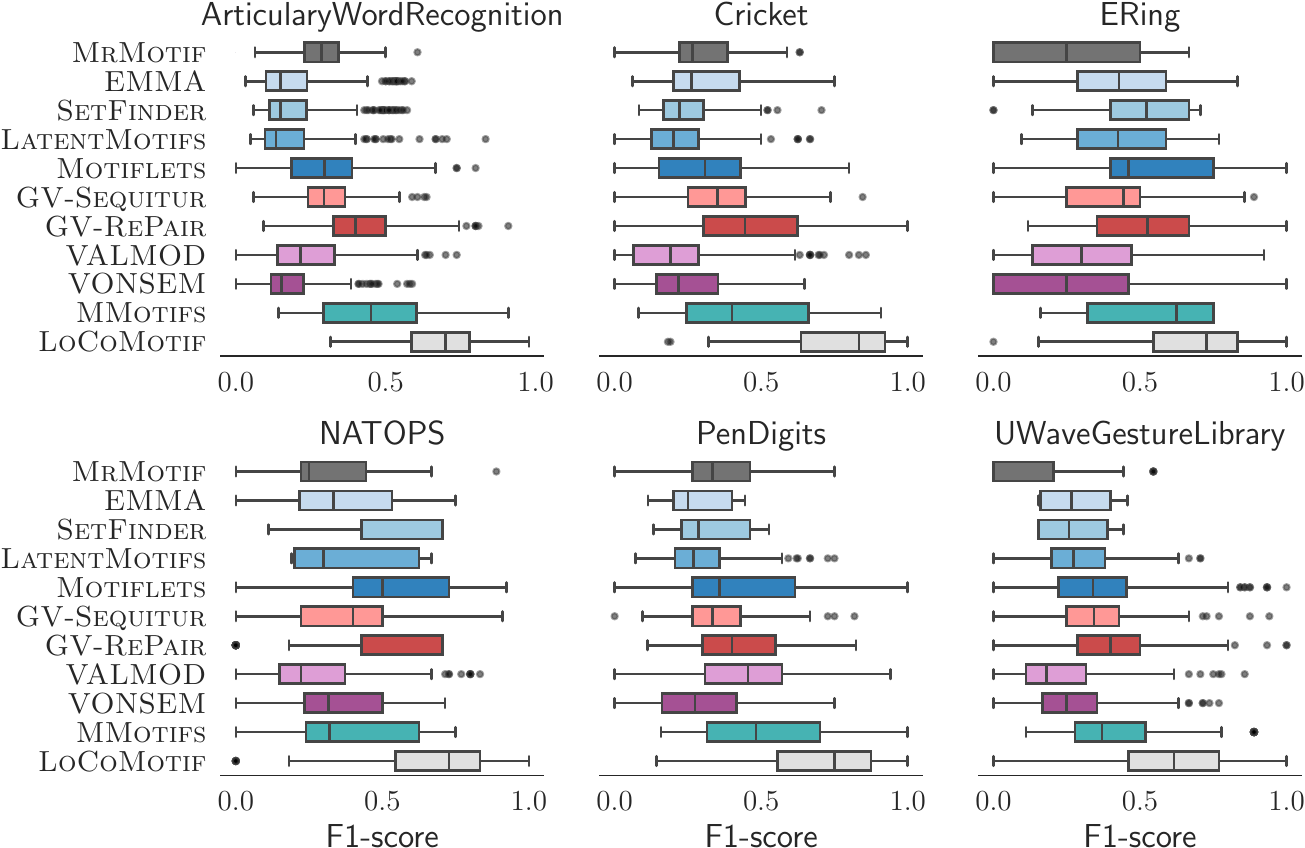}
    \caption{}\label{fig:boxplots-b}
\end{subfigure}
\end{figure}
\begin{figure}\ContinuedFloat
\begin{subfigure}[t]{\linewidth}\centering
    \includegraphics[width=\linewidth]{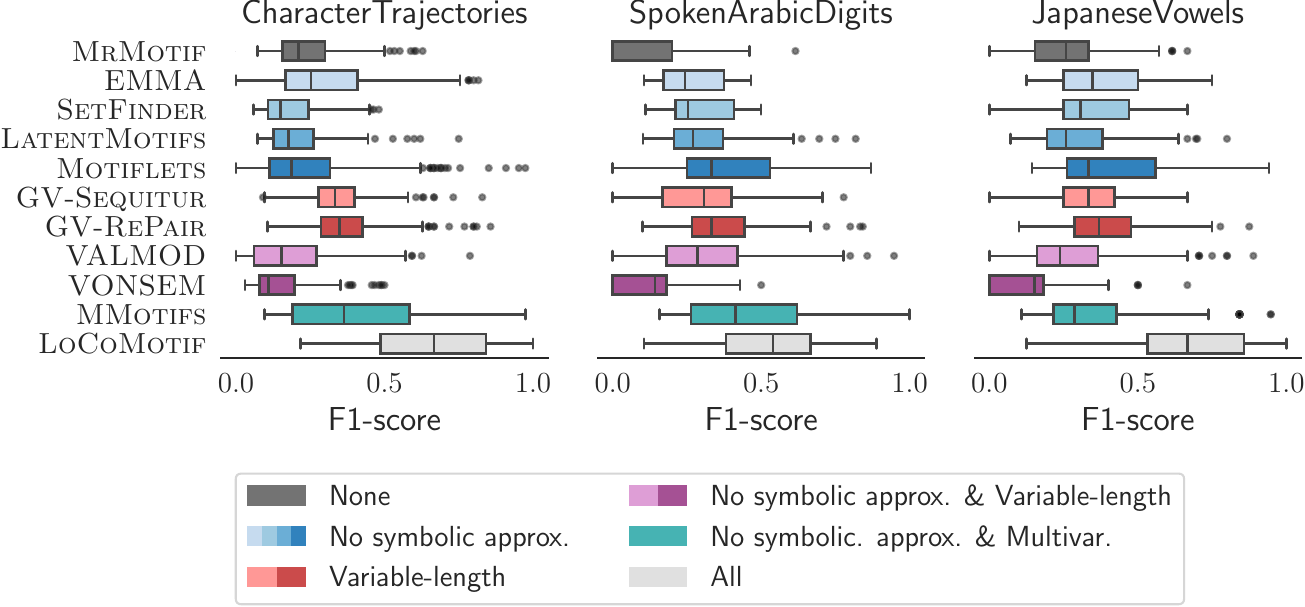}
\caption{}\label{fig:boxplots-c}
\end{subfigure}
\caption{The distribution of the F1-scores obtained by each of the applied methods for each benchmark dataset: (a) fixed-length, univariate datasets (b) fixed-length multivariate, datasets (c) variable-length, multivariate datasets. The hue of a method's box is based on which properties the method has in Table~\ref{tab:overview_methods}.}\label{fig:boxplots}
\end{figure}

Figure~\ref{fig:boxplots} shows, for each benchmark dataset, the distribution of the obtained F1-scores by each method. To answer \textbf{RQ2}, we perform an average rank comparison based on the average F1-score obtained for each dataset. To this end, we construct a critical difference diagram (CDD) using the statistical tests suggested by \cite{benavoli16a} with significance level $\alpha = 0.05$ (Figure \ref{fig:cdd-f}). Over all benchmark datasets, \textsc{LoCoMotif} statistically outperforms the other methods, achieving an average rank almost equal to 1. This is somewhat expected, as it checks all boxes in Table~\ref{tab:overview_methods}. It is followed by \textsc{MMotifs}, \textsc{GV-RePair}, and \textsc{Motiflets}, which all have an average rank around 3.5, but do not statistically outperform the other methods. Somewhat surprisingly, \textsc{VALMOD} and \textsc{VONSEM} are at the lower end of the ranking, despite them having 2 out of 3 capabilities in Table~\ref{tab:overview_methods}. This may be due to the absence of inter-variable-length motifs in our benchmark, which they are specialized for. 

We also construct a CDD for precision (Figure \ref{fig:cdd-p}) and recall (Figure \ref{fig:cdd-r}). From these diagrams, we observe that \textsc{LoCoMotif} distinguished itself from the rest in terms of F1-score by achieving higher precision, rather than a higher recall. Notable methods are \textsc{SetFinder}, which performed relatively well in terms of recall but not in terms of precision; and \textsc{GV-Sequitur}, for which the opposite holds. The CDD of precision (Figure \ref{fig:cdd-r}) also shows that the methods that use a symbolic approximation do not necessarily perform worse in terms of precision: \textsc{MrMotif}, \textsc{GV-Sequitur}, and \textsc{GV-RePair} perform relatively well in this CDD, and better than in the CDD of F1-score. This is in conflict with what is sometimes claimed about them \citep{motiflets, vonsem}. 

\begin{figure}
    \begin{subfigure}[t]{\linewidth}
    \centering
    \includegraphics[width=0.7\linewidth]{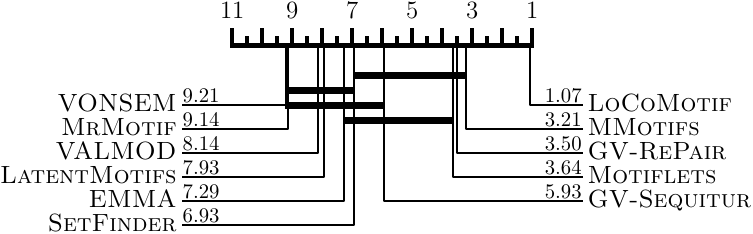}
        \caption{}\label{fig:cdd-f}
    \end{subfigure}
    \begin{subfigure}[t]{0.49\linewidth}
        \includegraphics[width=1\linewidth]{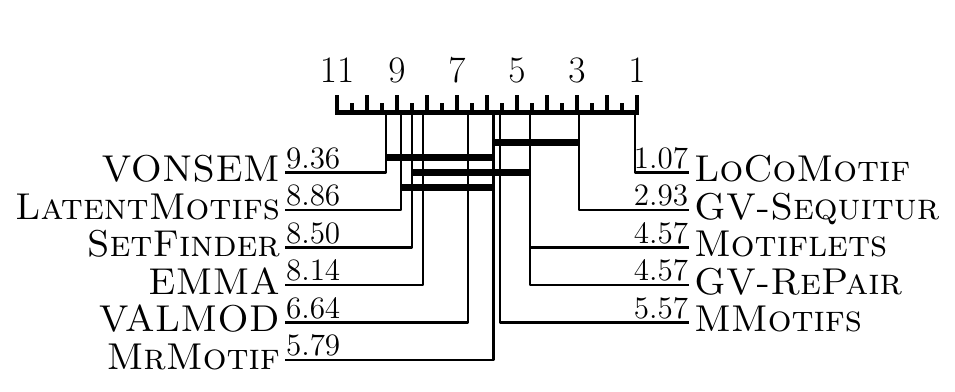}
        \caption{}\label{fig:cdd-p}
    \end{subfigure}\hfill
    \begin{subfigure}[t]{0.49\linewidth}
        \includegraphics[width=1\linewidth]{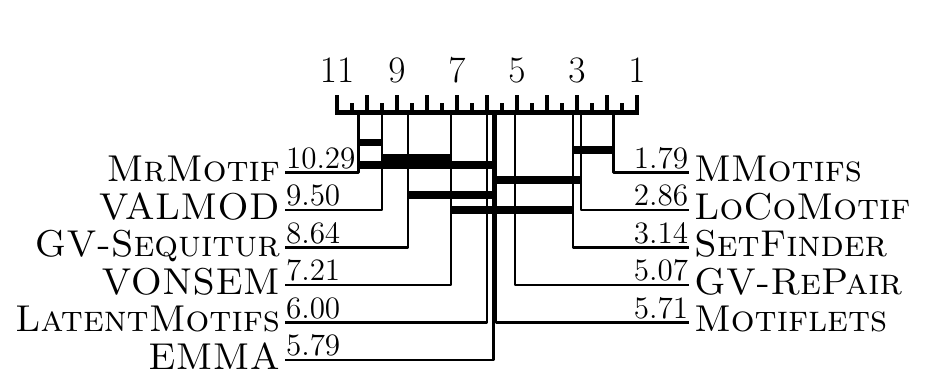}
        \caption{}\label{fig:cdd-r}
    \end{subfigure}
    \caption{Critical difference diagram based on the average (a) F (b) P, and (c) R obtained for each benchmark dataset.}
\end{figure}

\subsection{Showing that benchmarks based on random-walk are too easy (RQ3)}
\noindent In the Introduction (Section~\ref{sec:intro}) we argued that TSMD benchmarks that use random-walk data to interleave the GT motifs are trivial to solve, and therefore not representative of realistic TSMD tasks. Here, we show this by proposing a simple method that solves such benchmarks. Given a time series, the method identifies the time segments that include random-walk data, and returns the remaining time segments as motifs. 

The method makes use of the fact that the first difference ($r_{t+1} - r_{t}$) of a random-walk time series $\mathbf{r}$ is \emph{stationary} (i.e., its statistical properties such as mean, variance, etc. are all constant over time), whereas for structured time series classification instances this is typically not the case. Over a sliding window, the method calculates the first difference of the time series in the window and performs a Dickey-Fuller test \citep{dickeyfullertest} to determine whether the result is non-stationary ($H_0$) or stationary ($H_1$). If it is found to be stationary ($H_0$ rejected), the time series in the window is considered to be a random walk (rather than a classification instance). The windows without random walk are combined and returned as motifs.

To show that this simple method works, we generate 20 time series that consist of Gaussian random walk, and insert 2 to 8 instances from the \textsf{Symbols} dataset at random locations in each of them (see Figure~\ref{fig:rw-example} for an example). We then apply our simple method, with the window size equal to half the length of a \textsf{Symbols} instance and a significance level of 0.05. The method achieves an average F1-score of 0.98, which suggest that these type of benchmarks are indeed too easy, and that TSMD-Bench provides added value over them (answering \textbf{RQ3}).

\begin{figure}
    \centering
    \includegraphics[width=\linewidth]{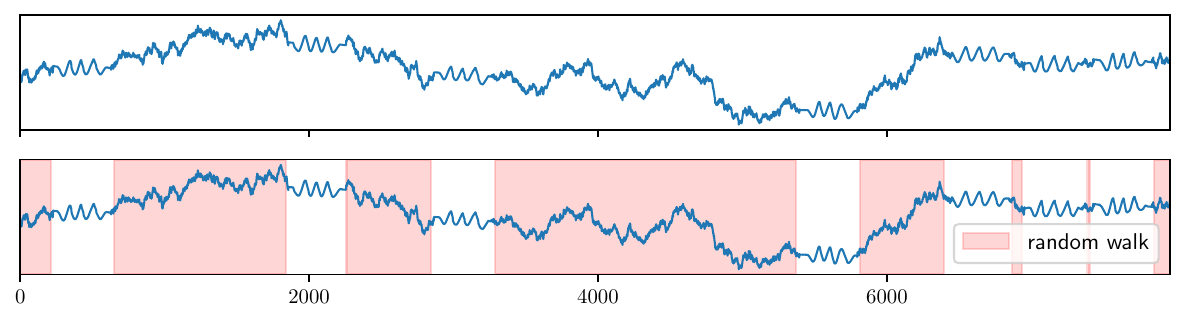}
    \caption{A random walk time series into which time series classification instances from the \textsf{Symbols} dataset are inserted. Such a time series is trivially solved by detecting the random walk segments, and returning the undetected segments as motifs.}
    \label{fig:rw-example}
\end{figure}

\section{Conclusion}
We have proposed PROM, a method to evaluate TSMD in a quantitative manner. Contrary to the existing metrics, PROM does not assume a specific TSMD setting, and measures both the recall and precision of the discovered motif sets with respect to the ground truth. Next to PROM, we proposed the benchmark TSMD-Bench in which ground truth is available.

We used PROM and TSMD-Bench to experimentally show that the existing evaluation metrics are biased towards either precision or recall. In addition, we demonstrated that PROM, combined with the benchmark, enables a fair and systematic comparative evaluation of TSMD methods, whereas no existing metric allowed such comparisons before. Finally, we showed that benchmarks with random-walk are too easy. \\ 

\noindent \textbf{Acknowledgements} This research received funding from the Flemish Government under the ``Onderzoeksprogramma Artificiële Intelligentie (AI) Vlaanderen'' programme and the VLAIO ICON-AI Conscious project (HBC.2020.2795).

\bibliography{references}
\newpage 

\begin{appendices}
\renewcommand{\thefigure}{\arabic{figure}}

\section{Macro-averaged P, R and F}\label{appendix:macro-averaging}
By default, PROM uses micro-averaging, which gives more weight to motif sets with a greater cardinality. If each motif set should have an equal weight, macro-averaging should be used instead.
In this case, $\mathrm{P}$ and $\mathrm{R}$ are calculated for each motif set separately:
\[
\text{P}_j = \frac{\text{TP}_j}{\text{TP}_j + \text{FP}_j}, \qquad
\text{R}_i = \frac{\text{TP}_i}{\text{TP}_i + \text{FN}_i}
\]
and then averaged:
\begin{equation*}
    \text{R} = \frac{\sum^{\ngt}_{i=1} \text{R}_i}{\ngt},
\end{equation*} 
\begin{equation*}
    \text{P} = \frac{\sum^{\nmin}_{j=1} \text{P}_j}{\nd} \ \ \text{or} \ \ \text{P} = \frac{\sum^{\nmin}_{j=1} \text{P}_j}{\nmin}
\end{equation*}
where the left formula is relevant if off-target motif sets are to be penalized ($\mathrm{P}_j$ of an off-target motif set is 0), and the right formula is relevant if they are not. $\mathrm{F}$ requires special care: $\mathrm{F}_i$ is only defined for matched motif sets (only for them both $\mathrm{P}_{i}$ and $\mathrm{R}_i$ are defined), but averaging $\mathrm{F}_i$ over only the matched motif sets does not penalize unmatched GT motif sets when $g>d$, and does not provide the option to penalize off-target motif sets when $d>g$. Instead, we define $\mathrm{F}_i = 0$ for unmatched GT motif sets and $\mathrm{F}_j = 0$ for unmatched discovered motif sets, and macro-averaged F as

\begin{equation*}
    \text{F} = \frac{\sum^{\ngt}_{i=1} \text{F}_i}{\ngt}.
\end{equation*}
when $g > d$, and 
\begin{equation*}
    \text{F} = \frac{\sum^{\nd}_{j=1} \text{F}_j}{\nd} \ \ \text{or} \ \ \text{F} = \frac{\sum^{\nmin}_{j=1} \text{F}_j}{\nmin}.
\end{equation*}
when $d > g$, depending on whether the unmatched discovered motif should be penalized or not.

\end{appendices}
\end{document}